\documentclass[preprint,12pt]{elsarticle}
\usepackage{latexsym}
\usepackage{amssymb}
\usepackage{amsmath}
\usepackage[]{graphicx}
\usepackage[usenames]{color}
\usepackage{amssymb}
\usepackage{enumitem}
\usepackage{framed,color}
\usepackage[algo2e,ruled,vlined,lined,linesnumbered]{algorithm2e}
\usepackage[algo2e]{algorithm2e} 
\usepackage{algorithm} 
\usepackage{algorithmic}   
\usepackage{multicol}
\usepackage{multirow}
\usepackage{pdfpages}
\usepackage{setspace}
\usepackage[table]{xcolor}
\definecolor{shadecolor}{rgb}{1,0.8,0.3}

\def \tuple#1{\langle #1 \rangle} 

\def\defemb#1#2{\expandafter\def\csname #1\endcsname
                              {\relax\ifmmode #2\else\hbox{$#2$}\fi}}

\defemb{cA}{{\cal A}}
\defemb{cB}{{\cal B}}
\defemb{cC}{{\cal C}}
\defemb{cD}{{\cal D}}
\defemb{cE}{{\cal E}}
\defemb{cF}{{\cal F}}
\defemb{cG}{{\cal G}}
\defemb{cH}{{\cal H}}
\defemb{cI}{{\cal I}}
\defemb{cJ}{{\cal J}}
\defemb{cK}{{\cal K}}
\defemb{cL}{{\cal L}}
\defemb{cM}{{\cal M}}

\defemb{cO}{{\cal O}}
\defemb{cP}{{\cal P}}
\defemb{cQ}{{\cal Q}}
\defemb{cR}{{\cal R}}
\defemb{cS}{{\cal S}}
\defemb{cT}{{\cal T}}
\defemb{cU}{{\cal U}}
\defemb{cV}{{\cal V}}
\defemb{cX}{{\cal X}}
\defemb{cW}{{\cal W}}

\long\def\comment#1{}

\defemb{ua}{{\underline{a}}}
\defemb{oa}{{\overline{a}}}
\defemb{ub}{{\underline{b}}}
\defemb{ob}{{\overline{b}}}

\newcommand{\umu}{\underline{\mu}}
\newcommand{\omu}{\overline{\mu}}
\newcommand{\ualpha}{\underline{\alpha}}
\newcommand{\oalpha}{\overline{\alpha}}
\newcommand{\ubeta}{\underline{\beta}}
\newcommand{\obeta}{\overline{\beta}}
\newcommand{\ulambda}{\underline{\lambda}}
\newcommand{\olambda}{\overline{\lambda}}

\def\defemb#1#2{\expandafter\def\csname #1\endcsname
                              {\relax\ifmmode #2\else\hbox{$#2$}\fi}}

\defemb{cA}{{\cal A}}
\defemb{cB}{{\cal B}}
\defemb{cC}{{\cal C}}
\defemb{cD}{{\cal D}}
\defemb{cE}{{\cal E}}
\defemb{cF}{{\cal F}}
\defemb{cG}{{\cal G}}
\defemb{cH}{{\cal H}}
\defemb{cI}{{\cal I}}
\defemb{cJ}{{\cal J}}
\defemb{cK}{{\cal K}}
\defemb{cL}{{\cal L}}
\defemb{cM}{{\cal M}}
\defemb{cO}{{\cal O}}
\defemb{cP}{{\cal P}}
\defemb{cQ}{{\cal Q}}
\defemb{cR}{{\cal R}}
\defemb{cS}{{\cal S}}
\defemb{cT}{{\cal T}}
\defemb{cU}{{\cal U}}
\defemb{cV}{{\cal V}}
\defemb{cX}{{\cal X}}

\defemb{ua}{{\underline{a}}}
\defemb{oa}{{\overline{a}}}
\defemb{ub}{{\underline{b}}}
\defemb{ob}{{\overline{b}}}

\newtheorem{definition}{Definition}[section]
\newtheorem{theorem}[definition]{Theorem}

\newtheorem{proposition}[definition]{Proposition}

\newtheorem{example}{Example}
\newtheorem{proof}{Proof}

\journal{Summited to Studies of Computational Intelligent Series}

\begin{document}

\begin{frontmatter}

\title{On the incorporation of interval-valued fuzzy sets into the Bousi-Prolog system: declarative semantics, 
implementation and applications}

\author[a1]{Clemente Rubio-Manzano}
\ead{clrubio@ubiobio.cl}
\author[a2]{Mart\'in Pereira-Fari{\~n}a}

\address[a1]{Dep. of Information Systems, University of the B\'{\i}o-B\'{\i}o, Chile.}
\address[a2]{Centre for Argument Technology, University of Dundee, UK. Departamento de Filosof\'ia e Antropolox\'ia, Universidade de Santiago de Compostela,Spain.}

\begin{abstract}
In this paper we analyse the benefits of incorporating interval-valued fuzzy sets into the Bousi-Prolog system. A syntax, declarative semantics and implementation for this extension is presented and formalised. We show, by using potential applications, that fuzzy logic programming frameworks enhanced with them can correctly work together with lexical resources and ontologies in order to improve their capabilities for knowledge representation and reasoning.\\
{\bf Keywords:} 
Interval-valued fuzzy sets, Approximate Reasoning, Lexical Knowledge Resources, Fuzzy Logic Programming, Fuzzy Prolog.
\end{abstract}


\end{frontmatter}
\section{Introduction and Motivation}
\label{sect:introduction}
Nowadays, lexical knowledge resources as well as ontologies of concepts are widely employed for modelling domain independent knowledge \cite{Miller1995,LS11} or by automated reasoners~\cite{Bob12b}. In the case of approximate reasoning, this makes possible to incorporate general knowledge into any system, which is independent of the programmer's background~\cite{RM16}.

   
Inside the former and current frameworks of fuzzy logic programming ~\cite{JO04,Jul16,Jul05,MPH11,Str08,Van09}, we argue that lexical reasoning might be an appropriate way for tackling this challenge, because of this type of knowledge is usually expressed linguistically. However, from a computational point of view, this source of information involves vagueness and uncertainty and, consequently, it must be specifically addressed. Fuzzy set theory (FS) is a good candidate, but it shows some particular limitations to this aim: i) sometimes, words mean different things to different people and this generates and additional layer of uncertainty that cannot be adequately handled by FS; ii) the definition of membership functions for word meaning is also a debatable question and, therefore, achieving an agreement by means of a standard fuzzy set it is difficult; and, iii) with respect to semantic similarity measures used in this proposal, there is not a dominant one and, therefore, for two given words, different degrees of resemblance can be obtained with the resulting additional level of uncertainty.\par

In the specific field of fuzzy logic programming and fuzzy Prolog systems, little attention has been paid to the impact of this type of high degree of uncertainty and vagueness inherent to lexical knowledge, which is used in the definition of knowledge bases and inference processes. Next, a very simple example is introduced in order to illustrate i) and ii) in the building of a Prolog knowledge base.

\begin{example} \label{example1}
\textit{Suppose that we extract from Internet two people's opinions about a particular football player. The first one says ``a is a normal player'' and the second one says ``a is a bad player''. If we consider the label for qualifying the highest quality (e.g., ``good'') as a basic component, this lexical knowledge could be modelled by using two annotated facts as: ``football\_player(a,good):-0.8.'' and ``football\_player(a,good):-0.6.'', respectively. In this case, we use ``football\_player(a,good):-0.6.'' given the infimum is usually employed. However, as it can be observed, the information of the first person is lost.}
\end{example}

Case iii) deserves a special attention, given it involves the use of independent linguistic 
resources (such as WordNet Similarity \cite{Ped04}). As we said, this tool provide us different measures according to alternative criteria for assessing the degree of similarity between two words. In Example~\ref{example2}, we illustrate this situation by means of a simple case.

\begin{example} \label{example2}
\textit{Suppose we have the fact ``loves(a,b)'' and we extract the closeness between ``loves'' and ``desires'' by using two different semantics measures obtaining $0.8$ and $0.6$. Therefore, in order to represent this semantic knowledge we could employ two facts either ``desires(a,b):-0.8'' or ``desires(a,b):-0.6''.}
\end{example}

In order to address both Examples 1 and 2 inside the same frame, we propose to enhance the Bousi-Prolog system with interval-valued fuzzy sets (IVFSs), since they allow us to capture the uncertainty associated to lexical knowledge better than FS. Several advantages have pointed out for dealing with environments with high uncertainty or imprecision using IVFSs, such as~\cite{Tur89}; other authors have also shown that IVFSs can generate better results than standard FSs~\cite{Bus10}. Additionally, the use of intervals for describing uncertain information has been successfully applied in the realms of decision making, risk analysis, engineering design, or scheduling \cite{Medina10}.

Both Example~\ref{example1} and Example~\ref{example2} can be easily modelled by means of IVFSs, using and interval for combining information of the different sources into a single fact such as ``football\_player(a,good):-[0.6,0.8]'' or ``desires(a,b):-[0.6,0.8]'', respectively.

The main contribution of this paper is to design and implement an interval-valued fuzzy logic language, and to incorporate it into the Bousi-Prolog system~\cite{RJ14}. This task involves different challenges both from theoretical and implementation points of view. The former entails adding a IVFSs arithmetic into the Warren Abstract Machine based on Similarity (SWAM)~\cite{JR09}; the latter, means to establish a (model-theoretic) declarative semantics for the language in the classical way, formalising the notion of least interval valued fuzzy Herbrand model for interval-valued fuzzy definite programs. 

This paper is divided into the following sections: section~\ref{sc:prel} introduces the concepts that support our approach; section~\ref{sc:syntax} describes the details of the syntax, semantics and implementation of the proposed language; section~\ref{sc:applications} analyses different realms where this programming language can be applied; in section~\ref{sc:rc}, the main differences between this proposal an others that are described in the literature are discussed; and, finally, section~\ref{sc:conclusions} summarizes our main conclusions and some ideas for future work.

\section{Preliminary Concepts}
\label{sc:prel}

\subsection{Interval-Valued Fuzzy Sets}
IVFSs are a fuzzy formalism based on two membership mappings instead of a single one, like in standard FSs. Each one of these membership functions are called, \emph{lower membership function} and \emph{upper membership function}. Both are established on a universe of discourse $\cX$, and they map each element from $\cX$ to a real number in the $[0,1]$ interval, where the elements of $\cX$ belongs to $\cA$ according to an interval.

\begin{definition}\label{def1}
An interval-valued fuzzy set A in $\cX$ is a (crisp) set of ordered triples: $\cA= \{ (x,\umu_{A}(x),\omu_{A}(x)): x \in \cX; \umu_{A}(x),\omu_{A}(x): \cX \rightarrow [0,1] \}$ where: $\umu,\omu$ are the lower and the upper membership functions, respectively, satisfying the following condition: $0 \leq \umu_{A}(x) \leq \omu_{A}(x) \leq 1\ \forall x \in \cX$
\end{definition}
\noindent

As can be observed in Definition~\ref{def1}, those intervals are included in $[0,1]$ and closed at both ends. On the other hand, some arithmetic operations on interval-numbers have been recalled since they are useful in operating on cardinalities of IVFSs. Let a=$[\ua,\oa]$, b=$[\ub,\ob]$ be intervals in $R$, and $r \in R+$. The arithmetic operations '+', '-', '$\cdot$' and power are defined as follows:

\begin{align}
    [\ua,\oa] + [\ub,\ob]&=[\ua + \ub, \oa + \ob];\\
    [\ua,\oa] - [\ub,\ob]&=[\ua - \ob, \oa - \ub];\\
    [\ua,\oa] \cdot [\ub,\ob]&=[min(\ua \cdot \ub, \ua \cdot \ob,  \oa \cdot \ub, \oa \cdot \ob ),max(\ua \cdot \ub, \ua \cdot \ob,  \oa \cdot \ub, \oa \cdot \ob)];\\
    ([\ua,\oa])^{r}&=[\ua^{r},\oa^{r}] \text{ for non-negative }\oa,\ua
\end{align}

The operations of union and intersection for IVFSs are defined by triangular norms. Let A, B be IVFSs in $\cX$, $t$ a t-norm and $s$ a t-conorm. The union of A and B is the interval-valued fuzzy set $A \cup B$ with the membership function: $\mu_{A \cup B}(x)=[s(\umu_{A}(x),\umu_{B}(x)),s(\omu_{A}(x),\omu_{B}(x))]$. The intersection of A and B is the IVFSs A$\cap$B in which $\mu_{A \cap B}(x)=[t(\umu_{A}(x),\umu_{B}(x)),t(\omu_{A}(x),(\omu_{B}(x))]$. Thus, de Morgan's laws for IVFSs A,B in $\cX$ are: $(A \cup B)^{c}= A^{c} \cap B^{c}$ and $(A \cap B)^{c}= A^{c} \cup B^{c}$.

Let $L$ be a lattice of intervals in $[0,1]$ that satisfies: 
\begin{align}
  L={ [x_1,x_2] \in [0,1]^2 \hspace{0.1cm} with \hspace{0.1cm}  x_1 \le x_2 }; \\
  [x_1,x_2] \le_{L} [y_1,y_2] \hspace{0.1cm} iff \hspace{0.1cm}  x_1 \le y_1 \hspace{0.1cm} and \hspace{0.1cm} x_2 \le y_2. 
\end{align}
Also by definition: 
\begin{align}
[x_1,x_2] <_{L} [y_1,y_2] \Leftrightarrow x_1 < y_1, x_2 \le  y_2\ or\ x_1 \le y_1, x_2 < y_2; \\
[x_1,x_2] =_{L} [y_1,y_2] \Leftrightarrow x_1 = y_1, x_2 = y_2. 
\end{align}

Hence, $0_{L}=[0,0]$ and $1_{L}=[1,1]$ are the smallest and the greatest elements in $L$.

\subsection{Approximate Deductive Reasoning}

When we consider a collection of imprecise premises and a possible imprecise conclusion inferred from them in a Prolog program, we are applying a process of approximate deductive reasoning. These set of statements can be interpreted under two different frames~\cite{Pereira2014b} in a Prolog program: conditional and set-based interpretations. If the former is assumed, an imprecise premise is an assertion qualified by a degree of truth; e.g. ``John is tall with [0.2,0.5]'' means that the degree of truthfulness of this sentence using and IVFS is  $[0.2,0.5]$. On the other hand, if the latter is adopted, the interval that qualifies the sentence means the degree of membership of an element to a specific set; e.g., ``John is tall with [0.2,0.5]'' means that the membership of John to the set of tall people is $[0.2,0.5]$. The conclusion inferred from an imprecise premise must be also qualified by the same type of degree; e.g. ``John is a good player with [0.2,0.5]''.\par

In order to preserver the coherence with classical Prolog, we adopt the propositional interpretation (the interval indicates the degree of truth of the assertion) and, consequently, approximate deductive reasoning is based on multi-valued modus ponens \cite{Hak98}: 

\begin{align}
& Q,[\ualpha,\oalpha] \\ 
& A \leftarrow Q ,[\ubeta,\obeta] \\
& A,T([\ualpha,\oalpha],[\ubeta,\obeta]]
\end{align}
If we have (9) and (10), we can deduce (11) with $T$ a t-norm defined on the lattice $L([0,1])$.

\section{Simple Interval-valued fuzzy prolog: syntax, semantics and implementation}
\label{sc:syntax}

The design of a programming language involves three main steps. Firstly, the definition of the syntax; secondly, the elaboration of a formal study of its semantics; and thirdly, an implementation of the system. In order to address the tasks related with syntax and semantics, we will follow the guidelines established in~\cite{LLoyd87} and \cite{JR09PPDP}\footnote{We assume familiarity with the theory and practice of logic programming.}; for the implementation task, we will follow the guidelines detailed in~\cite{JR09}.

\subsection{Sintax}

An Interval-valued fuzzy program conveys a classical Prolog knowledge base and a set of IVFSs, which are used for annotating the facts by means of an interval-valued fuzzy degree: $p(t_1,\ldots,t_n)[\ualpha,\oalpha]$.

\begin{definition} 
An interval-valued fuzzy definite clause is a Horn clause of the form $A [\ualpha,\oalpha]$ or  $A \leftarrow B_1 ,\ldots, B_n [\ubeta,\obeta]$, where $A$ is called the head, and $B_1 , \ldots , B_n$ denote a conjunction which is called the body (variables in a clause are assumed to be universally quantified). 
\end{definition}

\begin{definition}
An interval-valued fuzzy definite program is a finite set of interval-valued fuzzy clauses.
\end{definition}

\begin{example}\label{ex-alphabet}

Let $\Pi=\{p(X) \leftarrow q(X), q(a)[0.8,0.9], q(b)[0.7,0.8]\}$ be an interval-valued fuzzy definite program, $\Pi$ generates a first order language, $\cL$, whose alphabet is comprised of the set of variable symbols, $\cX$, constant symbols, $\cC$, function symbols, $\cF$ and  predicate symbols, $\cP$, which appear in the clauses of $\Pi$. We assume that the first  order language $\cL$ has, at least one constant symbol; i.e., an assertion. If there are not constants available in the alphabet, an artificial constant ``$a$'' must be added to it. The first order language $\cL$ generated by $\Pi$ is: ${\cal X}=\{x\}$, ${\cC}=\{a,b,c\}$, ${\cal F}=\emptyset$ and ${\cal P}=\{p,q,r\}$.
\end{example}

\subsection{Declarative Semantics}
In logic programming, the declarative semantics for a program is traditionally formulated on the basis of the least Herbrand model (conceived as the infimum of a set of interpretations). In this section, we formally introduce the semantic notions of Herbrand interpretation, Herbrand model and least Herbrand model for an interval-valued fuzzy program $\Pi$, in order to characterise it.
 
In our framework, truth-values of the facts are modelled in terms of interval-valued degrees $[\ualpha,\oalpha]$ with $0 \leq \ualpha \leq \oalpha \leq 1$. An interval-valued fuzzy interpretation $\cI$ is a pair $\tuple{\cD, \cJ}$ where $\cD$ is the domain of the interpretation and $\cJ$ is a mapping which assigns meaning to the symbols of $\cL$: specifically n-ary relation symbols are interpreted as mappings $\cD^n \longrightarrow L([0,1])$. In order to evaluate open formulas, we have to introduce the notion of variable assignment. A variable assignment, $\vartheta$, w.r.t. an interpretation $\cI = \tuple{\cD, \cJ}$, is a mapping $\vartheta: \cV \longrightarrow \cD$, from the set of variables $\cV$ of $\cL$ to the elements of the interpretation domain $\cD$. This notion can be extended to the set of terms of $\cL$ by structural induction as usual. The following definition formalises the notion of {\em valuation of a formula} in our framework.

\begin{definition}
Given an interval-valued fuzzy interpretation  $\cI$ and a variable assignment $\vartheta$ in $\cI$, the valuation of a formula w.r.t. $\cI$ and $\vartheta$ is: 
\begin{enumerate}
\item 
    \begin{enumerate}
        \item $\cI(p(t_1,\ldots, t_n))[\vartheta] =  \bar{p}(t_1\vartheta, \ldots, t_n\vartheta),
\mbox{ where } \cJ(p)=\bar{p}$;
        \item $\cI(A_1, \ldots,  A_n))[\vartheta] = inf \{\cI(A_1)[\vartheta],\ldots, \cI(A_n)[\vartheta]\}$;
    \end{enumerate}
\item $\cI(A \leftarrow \cQ)[\vartheta] =  1$ if $I(A) >= I(Q)$; $\cI(A \leftarrow \cQ)[\vartheta] =  \cI(\cA)[\vartheta]$ if $I(A) < I(Q)$;
\item $\cI((\forall x)\cC)[\vartheta] =  \inf\{\cI(\cC)[\vartheta']\mid \vartheta'~~\mbox{$x$--equivalent to}~~ \vartheta\}$ where $p$ is a predicate symbol, $A$ and $A_i$ atomic formulas and $\cQ$ any body, $\cC$ any clause, $T$ is any left-continuous t-norm defined on $L([0,1])$. An assignment $\vartheta'$ is $x$--equivalent to $\vartheta$ when $z\vartheta' = z\vartheta$ for all variables $z\neq x$ in $\cV$. 
\end{enumerate}
\end{definition}

\begin{definition}
	Let $\cL$ be a first order language. The {\em Herbrand universe} $\cU_{\cL}$ for $\cL$, is the set of 	all ground terms, 	which can be formed out of the constants and function symbols appearing in $\cL$. 
\end{definition}
\begin{definition}
	Let $\cL$ be a first order language. The {\em Herbrand base} $\cB_{\cL}$ for $\cL$ is the set of all 	ground atoms which can be formed by using predicate symbols from $\cL$ with ground terms from the Herbrand universe as arguments.
\end{definition}
\begin{example}
	Let us consider again the language $\cL$ generated by the program $\Pi$ of Example~\ref{ex-alphabet}, 	the Herbrand universe $\cU_{\cL}=\{a,b,c\}$ and the Herbrand base: $\cB_{\cL}= \{p(a),p(b),p(c),q(a),q(b),q(c),r(a),r(b),r(c)\}$.
\end{example}

It is well-known that, in the classical case, it is possible to identify a Herbrand interpretation with a subset of the Herbrand base. Therefore, a convenient generalization of the notion of Herbrand interpretation to the interval-valued fuzzy case consists in establishing an interval-valued fuzzy Herbrand interpretation as an interval-valued fuzzy subset of the Herbrand base.

\begin{definition}[Interval-valued fuzzy interpretation]
Given, a first order language $\cL$, an interval-valued fuzzy Herbrand interpretation for $\cL$ is a mapping $\cI: \cB_{\cL} \longrightarrow L([0,1])$. 
\end{definition}

Hence, the truth value of a ground atom $A\in B_{\cL}$ is $\cI(A)$. Sometimes we will represent an interval-valued fuzzy Herbrand interpretation $\cI$ extensively: as a set of pairs $\{\tuple{A, [\ualpha,\oalpha]} \mid A \in B_{\cL} \mbox{ and } [\ualpha,\oalpha]=\cI(A)\}$.

Now, we introduce the notion of Interval-valued Fuzzy Herbrand Model, which is formalised in Definitions~\ref{def8} and~\ref{def9}. We employ a declarative semantics based on a threshold \cite{Ebr01,JR09PPDP}. Intuitively, a threshold $[\ulambda,\olambda]$ is delimiting truth degrees equal o greater that  $[\ulambda,\olambda]$ as true. Therefore, we are going to speak of Interval-valued Fuzzy Herbrand Model at level $[\ulambda,\olambda]$ or simply $[\ulambda,\olambda]$-model.

\begin{definition}\label{def8}
An Interval-valued fuzzy Herbrand Interpretation is a $[\ulambda,\olambda]$-model of an interval-valued fuzzy clause $\cC [\ualpha,\oalpha]$ if and only if $\cI(C) \geq [\ualpha,\oalpha] \geq [\ulambda,\olambda]$.
\end{definition}

\begin{definition}\label{def9}
An Interval-valued fuzzy Herbrand Interpretation is a $[\ulambda,\olambda]$-model of an interval-valued fuzzy program
$\Pi$ if and only if $\cI$ is a $[\ulambda,\olambda]$-model for each clause $\cC [\ualpha,\oalpha] \in \Pi$.
\end{definition}

\begin{theorem} \label{theorem_model_herbrand}
Let $\Pi$ be an Interval-valued fuzzy program and suppose $\Pi$ has a $[\ulambda,\olambda]$-model. Then $\Pi$ has 
a Herbrand $[\ulambda,\olambda]$-model.
\end{theorem}

\begin{proof}
	Suppose that $\cM$ is a $[\ulambda,\olambda]$-model of $\Pi$. Let $\cM'$ be an Interval-valued fuzzy Herbrand 
	interpretation: $\cM'=\{ A \in B_{\Pi} \mid \cM(A) \ge [\ulambda,\olambda] \}$. 
	We are going to prove that this interpretation is a $[\ulambda,\olambda]$-model for all clauses of $\Pi$. Let $\cC$ any 
	clause, by initial supposition 
	and by definition of $[\ulambda,\olambda]$-model for an interval-valued fuzzy program, we have that:
	\begin{center}
		$\cC \equiv \forall x_1, \ldots, x_n (p(x_1,\ldots ,x_n) \leftarrow 
		q_1(x_1,\ldots ,x_n) \wedge \ldots \wedge 
		q_m(x_1,\ldots ,x_n)) [\ubeta,\obeta] $
		
	\end{center}
	
	$\cM$ is a $[\ulambda,\olambda]$-model of $C$ iff $\forall a_1, \ldots, a_n \in U_{L}, \cM(C) \geq [\ubeta,\obeta] \geq [\ulambda,\olambda]$. Let $a_1, \ldots, a_n \in U_{L}$ then
	we have that $\cM(p(a_1,\ldots,a_n))=[\ubeta,\obeta] \geq [\lambda,\olambda]$ what implies 
	that $\cM'(p(a_1,\ldots,a_n)) \geq [\ulambda,\olambda]$
	
\end{proof}

\begin{definition} 
Let $\Pi$ be an interval-valued fuzzy program. Let $\cA$ be an interval-valued fuzzy clause of $\Pi$. 
Then $\cA$ is a logical consequence of $\Pi$ at level $[\ulambda,\olambda]$ if and only if for each interval-valued fuzzy interpretation $I$, if $I$ is a $[\ulambda,\olambda]$-model for $\Pi$ then $\cI$ is a $[\ulambda,\olambda]$-model for $\cA$.
\end{definition}

\begin{proposition}
$\cA$ is a logical consequence of an interval-valued fuzzy program $\Pi$ at level $[\ulambda,\olambda]$ if and only if for every interval-valued 
fuzzy Herbrand interpretation $\cI$ for $\Pi$, if $\cI$ is a $[\ulambda,\olambda]$-model for $\Pi$, it is an interval-valued 
fuzzy Herbrand $[\ulambda,\olambda]$-model for $A$.

\begin{proof}
       First, let us suppose that $\cA$ is a logical consequence 
       for $\Pi$ at level $[\ulambda,\olambda]$, then, by definition, for any interval-valued fuzzy interpretation $\cI$ if $\cI$ is $[\ulambda,\olambda]$-model for 
       $\Pi$, it is a $[\ulambda,\olambda]$-model for $\cA$. Moreover, by the Theorem~\ref{theorem_model_herbrand}, there must exist 
       $\cI$' which being an interval-valued fuzzy Herbrand model for $\Pi$ at level $[\ulambda,\olambda]$, it is a $[\ulambda,\olambda]$-model  for $\cA$. This establishes
       the first side of the argument. Now, we have that for every interpretation $\cI$, if $\cI$ is a Herbrand model for $\Pi$ at level $[\ulambda,\olambda]$, 
       it is a Herbrand $[\ulambda,\olambda]$-model for $\cA$. Let $\cM$ be an interpretation, not necessarily Herbrand, which is a $[\ulambda,\olambda]$-model for $\Pi$. 
       We have that: $\cM'=\{ p(t_1,\ldots,t_n)[\ualpha,\oalpha]$ with $p(t_1,\ldots,t_n) \in B_{L} \mid \cM(p(t_1,\ldots,t_n)) \geq 
       [\ualpha,\oalpha] \geq [\ulambda,\olambda] \}$
       and by the Theorem~\ref{theorem_model_herbrand} M' is a $[\ulambda,\olambda]$-model for $\Pi$. And so it is for A. So, M is a $[\ulambda,\olambda]$-model 
       for all ground instances A' of A. As result M is a $[\ulambda,\olambda]$-model for A', hence for A and A'. This
       establishes the other side of the argument.
	
\end{proof}
\end{proposition}

The ordering $\le$ in the lattice $L([0,1])$ can be extended to the set of interval-valued fuzzy interpretation as follows: 
$I_1 \sqsubseteq I_2$ iff $I_1(A) \le I_2(A)$ for all interval-valued fuzzy atom $A \in B_{L}$. It is important note that the
pair $\tuple{H^{IVF},\sqsubseteq}$ is a complete lattice. Then it comes equipped with t-norms and t-conorms, that is, $T(I_1,I_2)$
is an interval-valued fuzzy interpretation for all $A \in B_{L}$, and $t(I_1,I_2)$ an interval-valued fuzzy interpretation 
for all $A \in B_{L}$. Therefore, the top element of this lattice is $\tuple{A,[1,1]}$ with $A \in B_{L}$ and 
the bottom element is $\tuple{A,[0,0]}$ with $A \in B_{L}$. 

Interval-valued fuzzy interpretations have an important property which allow us to characterize the semantics 
of an interval-valued fuzzy program $\Pi$. 

\begin{definition}
If $M_1$ is a model of $\Pi$ at level $[\ulambda_1,\olambda_1]$ and $M_2$ is a model of $\Pi$ at level $[\ulambda_2,\olambda_2]$, then
$M_1 \cap M_2$ contains the interval-valued fuzzy atom in both $M_1$ and $M_2$ but to degree min($[\ulambda_1,\olambda_1]$, 
$[\ulambda_2,\olambda_2]$).	
\end{definition}

\begin{proposition}[Intersection Property of Models: Min-Model]
Let $\Pi$ be an interval-valued fuzzy program. Let $\cM_1,\ldots,\cM_n$ be a non-empty set of model 
for $\Pi$ at levels $[\ulambda_1,\olambda_1]  \ldots [\ulambda_n,\olambda_n]$, respectively. Then $\bigcap  (\cM_1,\ldots,\cM_n) \geq min([\ulambda_1,\olambda_1]  \ldots [\ulambda_n,\olambda_n])$ is 
a min-model for $\Pi$.

\begin{proof}
	We prove this proposition by induction on the number of interpretations $i$: 
	\begin{enumerate}
	    \item \textbf{Base Case (i=2).} Let $M_1$ and $M_2$ be models for $\Pi$ at levels 	$[\ulambda_1,\olambda_1]$ and $[\ulambda_2,\olambda_2]$. Then for all interval-valued fuzzy clause $\cC$, $M_1(C) \geq [\ulambda_1,\olambda_1]$ and $M_2(C) \geq [\ulambda_2,\olambda_2]$, so $M_1 \cap M_2$ is a min($[\ulambda_1,\olambda_1],[\ulambda_2,\olambda_2]$)-model for $\Pi$;
	    \item \textbf{Inductive Case (i=n).} Let $M_1, M_2, \ldots M_n$ be models for $\Pi$ at levels 		$[\ulambda_1,\olambda_1] \ldots [\ulambda_n,\olambda_n]$. Then for all 
		interval-valued fuzzy clause $C$, $M_i(C) \geq min([\ulambda_i,\olambda_i])$, so by the properties of the minimum.
	\end{enumerate}
\end{proof}
\end{proposition}

\begin{definition}
Let $\Pi$ be an interval-valued fuzzy program. The least model for $\Pi$ is defined as follows: $\cM=\bigcap \{ \cI(A) \geq [\ulambda,\olambda]  \mid A \in B_{L}  \}$. We call it a min-interval-valued fuzzy degree $[\ulambda,\olambda]_{min}$.
\end{definition}

\begin{theorem} \label{theorem_least_model}
 Let $\Pi$ an interval-valued fuzzy program. Let $\cM$ be the least  model of $\Pi$. Let $\cA \in \cB_{\cL}$ a ground atom of the interval-valued fuzzy Herbrand base. $\cM (\cA) \geq [\ulambda,\olambda]_{min}$ if and only if $\cA$ is logical consequence of $\Pi$ at level 
 $[\ulambda,\olambda]_{min}$.
 
 \begin{proof}
 First, by definition $\cM = \bigcap\{ I(A) \geq [\ulambda,\olambda] \mid A \in B_{L}  \}$. Hence, for all model I of $\Pi$, I(A) $\ge$ $\cM (A) \geq [\ulambda,\olambda]_{min}$. That is, $A$ is a logical consequence for $\Pi$ at level $[\ulambda,\olambda]_{min}$. This establishes the first side of the argument. Now,  If $A$ is a logical consequence of $\Pi$ by definition all model I for $\Pi$, I is a $[\ulambda,\olambda]_{min}$-model for A. That is, $I(A) \geq [\ulambda,\olambda]_{min}$. Therefore, by definition of least model, $\bigcap (I(A)) \geq [\ulambda,\olambda]$ what implies that $\cM \geq [\ulambda,\olambda]_{min}$. This establishes the another side of the argument.
 \end{proof}
\end{theorem}

\subsection{Fixpoint Semantics}

In this section, we give a deeper characterisation of the least Herbrand model for an interval-valued fuzzy program $\Pi$ using fixpoint concepts. 

This is possible because of each interval-valued fuzzy program has associated a complete lattice of interval-valued fuzzy Herbrand interpretations and we can define a continuous operator on that lattice. This allows us to provide a constructive vision of the meaning of a program by defining an immediate consequences operator and to construct the least Herbrand model by means of successive applications.

\begin{definition}[Fixpoint Characterization of the least Herbrand model]
Let $\Pi$ be an interval-valued fuzzy program, the mapping $O :2^{B_{\cL}} \rightarrow 2^{B_{\cL}}$ is defined as follows. Let $\cI$ be an interval-valued fuzzy Herbrand interpretation, then:
\begin{center}
$\cO =\{ \cA \in B_{\cL}: \cA \leftarrow \cB_1, \ldots, \cB_n [\ualpha,\oalpha]$ is a ground instance of a clause
in $\Pi$ and $\cI(\cB_i) \geq [\ualpha,\oalpha] \geq [\ulambda,\olambda]$ where $\cI(A) \geq inf(\cI(\cB_1,\ldots,\cB_n)) \}$
\end{center}
\end{definition}

As in the case of classical logic programming, interval-valued fuzzy Herbrand interpretations which are models can be characterised in terms of the operator $\cO$.

\begin{theorem} \label{proposicion2fp}
Let $\Pi$ be an interval-valued fuzzy program. Let $\cI$ be an interval-valued fuzzy Herbrand interpretation of $\Pi$. $\cI$ is $[\ulambda,\olambda]$-model for $\Pi$ if and only if $\cO(I) \subseteq \cI$.

\begin{proof}
$\cI$ is a $[\ulambda,\olambda]$-model for $\Pi$ if and only if for all clause $\cC$ in $\Pi$ then $\cI(\cC) \geq [\ulambda,\olambda]$. Therefore, it is fulfilled if and only if for every variable assignment $\vartheta$,  $\cI(\cC \vartheta) \geq [\ulambda,\olambda]$. Therefore, supposing without loss of generality that $\cC \equiv \cA \leftarrow \cB_1,\ldots,\cB_n [\ualpha,\oalpha] $ then $\cI(\cA \leftarrow \cB_1,\ldots,\cB_n \vartheta) \geq [\ulambda,\olambda]$, by the properties of the t-norm minimun $\cI(\cB_1,\ldots,\cB_n \vartheta) \geq [\ualpha,\oalpha] \geq [\ulambda,\olambda]_{min}$ what implies that $\cI(\cB_1,\ldots,\cB_n \vartheta) \subseteq \cO $ and hence $\cI(\cA \vartheta) \subseteq \cO(I)$, again by the properties of the t-norm minimun $\cI(\cA \vartheta) \geq  inf(\cI(\cB_1 \vartheta),\ldots, \cI(\cB_n \vartheta))$ what implies that $\cO(\cI) \subseteq \cI$
\end{proof}
\end{theorem}

Now we are ready to demonstrate the main theorem of this subsection, but first we recall the following results from fixpoint theory.

\begin{theorem}[FixPoint Theorem] \label{teoremafp}
Let $\tuple{L,\leq}$ be a complete lattice and $O: L \rightarrow L$ be a monotonic mapping. Then $O$ has a least fixpoint
$lfp(O)=inf \{ x \mid O(x)=x \}=inf \{ x \mid T(x) \leq x \}$.
\end{theorem}

\begin{proposition} \label{propositionfp}
Let $\tuple{L,\leq}$ be a complete lattice and $O: L \rightarrow L$ be a continuous mapping. Then $lfp(O) = O \uparrow \omega$.
\begin{proof}
See \cite{LLoyd87}
\end{proof}

\end{proposition}

\begin{theorem}
Let $\Pi$ be an interval-valued fuzzy definite program. 
Then $\cM = lfp(\cO_{\Pi}^{T}) =  \cO \uparrow \omega$.

\begin{proof}
$\cM$ is the least model which is the intersection of any $[\ulambda,\olambda]$-model for $\Pi$. As the lattice
of interval-valued fuzzy Herbrand models is a complete one, then we can use the Theorem~\ref{teoremafp}, the
Proposition~\ref{propositionfp} and the Theorem~\ref{proposicion2fp}. Applying them and the continuity of 
$\cO$ establishes the theorem.
\end{proof}

\end{theorem}

\begin{example}
Given the program $\Pi$ of Example~\ref{ex-alphabet}, 
the least Herbrand model for $\Pi$:

\begin{description}
 \item $O \uparrow 0 = I_{\bot}$; 
 \item $O \uparrow 1 = O(O \uparrow 0)=\{ \tuple{p(a),[0,0]}, \tuple{p(b),[0,0]}, \tuple{q(a),[0.8,0.9]}, \tuple{q(a),[0.7,0.8]} \}$
 \item $O \uparrow 2 = O(O \uparrow 1)=\{ \tuple{p(a),[0.8,0.9]}, \tuple{p(b),[0.7,0.8]}, \tuple{q(a),[0.8,0.9]}, \tuple{q(a),[0.7,0.8]} \}$
 \item $O \uparrow 3 = O \uparrow 2$
\end{description}
Therefore, as the fixpoint is reached at the next item: $\cM=O \uparrow 2$
\end{example}

\subsection{Operational Semantics} \label{sec:operationalsemantics}

We begin by providing definitions of an interval-valued SLD-derivation and an 
interval-valued fuzzy SLD-refutation that will be used later for showing the soundness and the completeness of the system.

\begin{definition} Let $\cG$ be $\leftarrow A_1, \ldots, A_m, \ldots, A_k$ and $C$ be either
$A [\ualpha,\oalpha]$ or $A \leftarrow B_1,\ldots,B_q [\ubeta,\obeta]$. 
Then $G'$ is derived from G and C using mgu $\theta$ if the following conditions 
hold (G' is the interval-fuzzy resolvent of G and C): i) $A_m$ is an atom called 
the selected atom in G; ii) $\theta$ is a mgu of $A_m$ and A; iii) $G'$ is the interval-valued 
fuzzy goal $\leftarrow (A_1, \ldots, B_1,\ldots,B_q, , \ldots, A_k)\theta$ with 
$[\ualpha_{G'},\oalpha_{G'}]=min([\ualpha_{C},\oalpha_{C}],[\ualpha_{G},\oalpha_{G}])$
\end{definition}

\begin{definition} 
 An interval-valued fuzzy SLD-derivation of $\Pi \cup G$ is a successful interval-valued SLD-derivation of $\Pi \cup G$ which 
 has the empty clause as the last goal in the derivation. If $G_n$ is the empty clause, we say that the derivation has 
 length $n$. The empty clause is derived from $\leftarrow (A_1, \ldots, A_m, \ldots, A_k) [\ualpha_{G},\oalpha_{G}]$ 
 and  $A(t_1,\ldots,t_q) [\ualpha_{A},\oalpha_{A}] \leftarrow$ 
 with $[\ualpha_{G_n},\oalpha_{G_n}]=min([\ualpha_{A},\oalpha_{A}], [\ualpha_{G},\oalpha_{G}]) $
\end{definition}

\begin{definition} 
 Let $\Pi$ be an interval-valued fuzzy program and $G$ be an interval-valued fuzzy goal. An interval-valued fuzzy computed 
 answer  $\tuple{\theta,[\ubeta,\obeta]}$ for $\Pi \cup G$ is the substitution obtained by restricting the 
 composition  $\theta_1, \ldots, \theta_n$ to the variables of $G$, where $\theta_1,\ldots,\theta_n$ is the sequence 
 of mgu's employed in the finite interval-valued fuzzy SLD-derivation of $\Pi \cup G$ with an interval-valued 
 approximation degree $[\ubeta,\obeta]$
\end{definition}

\begin{definition}
 Let $\Pi$ be an interval-valued fuzzy program, $G$ be an interval-valued fuzzy
 goal $\leftarrow (A_1, \ldots, A_k)$ and $\tuple{\theta,[\ubeta,\obeta]}$ be an answer 
 for $\Pi \cup G$. We say that $\tuple{\theta,[\ubeta,\obeta]}$ is an  interval-valued 
 fuzzy correct answer if $\forall (A_1, \ldots, A_k) \theta$ is a logical consequence  of $\Pi$ at level $[\ulambda,\olambda]_{min}$, that is, $[\ubeta,\obeta] \geq [\ulambda,\olambda]_{min}$.
\end{definition}

\subsection{Implementation} \label{implementation}
In this section, we briefly explain how interval-valued fuzzy sets are incorporated into the Bousi-Prolog system\footnote{A beta version can be founded at the URL:
http://www.face.ubiobio.cl/$\sim$clrubio/bousiTools/}. Here, we describe the structure and main features of its abstract machine. It was created as extension of the SWAM for the execution of Bousi-Prolog programs. We have appropriately modified the compiler, some machine instructions and SWAM structures in order to trigger the interval-valued fuzzy resolution. It is worth noting that, to the best of our knowledge, this is the first SWAM implementation that supports interval-valued fuzzy resolution.

A mandatory step to achieve this result is to include a new data structure into the architecture for computing with interval-valued fuzzy sets. This data structure has been implemented by using a class called {\tt IntervalFS} which is formed by two private attributes of double type: upper\_limit, lower\_limit. We define the public method constructor  {\tt IntervalFS(double ll,double lu)} and the four methods (sets and gets):  {\tt  double getUpperLimit(); double getLowerLimit(); void setUpperLimit(double v); void setLowerLimit(double v)}. Additionally, we overwrite both the \emph{toString} and the equals methods in the usual way. Finally, methods for adding, substracting and computing minimum of interval valued fuzzy set are implemented: IntervalFS add(IntervalFS a, IntervalFS b); IntervalFS substract(IntervalFS a, IntervalFS b);  IntervalFS min(IntervalFS a, IntervalFS b).

The following example illustrates the new features of the SWAM enhanced with IVFSs.

\begin{example}
Let us suppose that we want to represent the following knowledge: a football player is good when he is fast, tall and coordinated. We know a particular player that is fast, quite tall but he is not very coordinated. Thus, is he a good player? Answering this question and in this scenario, the linguistic expression ``is not very coordinate'' could be represented by the fact ``coordinate(a) [0.2,0.4]'', the linguistic term ``fast'' could be represented by the fact ``fast(a) [0.9,1.0]'' and ``quite tall'' could be represented by the fact ``tall(a)[0.8,0.9]''. A possible solution by employing a Bousi-Prolog program is described as follows:

\begin{verbatim}
%% FACTS 
coordinate(a) [0.2,0.4]
fast(a) [0.9,1.0]
tall(a )[0.8,0.9]
 
%% RULES 
good_player(X):-tall(X), fast(X), coordinate(X) 
\end{verbatim}

The SWAM enhanced with IVFSs allows us to obtain the answer: ``X=a with [0.2,0.4]''. The SWAM code generated for this program is as follows:

\begin{verbatim}
00:good_player:trust_me [1.0,1.0]   11:coordinate:trust_me   [0.2,0.4]
01:            allocate             12:           get_constant a A0 
02:            get_variable Y0 A0   13:           proceed   
03:            put_value Y0 A0      14:           fast: trust_me [0.9,1.0]
04:            call coordinate (11) 15:           get_constant a A0 
05:            put_value Y0 A0      16:           proceed   
06:            call fast   (14)     17:tall:      trust_me   [0.8,0.9]
07:            put_value Y0 A0      18:           get_constant a A0 
08:            call tall   (17)     19:           proceed  
09:            deallocate           20:query:     trust_me
10:            proceed              21:           create_variable Q0 X
                                    22:           put_value Q0 A0
                                    23:           call good_player (00)
                                    24:           halt
\end{verbatim}
\end{example}

The first instruction to be executed is the one labelled with the key ``query'', hence the execution starts at the position 20 with a degree $D=[1.0,1.0]$ (which is fixed in the instruction $trust\_me$). After that, from line 20 to line 23 the query is launched and the variable $X$ is created ($create\_variable$ instruction). After that from line 00 to line 04 the first subgoal (coordinate(X)) is launched, then the execution goes to line 11 and the unification with the term ''coordinate(a)´´ is produced (from line 11 to 13) ($put\_value$ and $get\_constant$ instructions), a new approximation degree is established $D=min([1.0,1.0],[0.2,0.4])$ ($trust\_me$ instruction), as these terms unify the following subgoal ($fast(X)$, line 05 and from line 14 to line 16) is launched with an approximation degree $D=min([0.2,0.4],[0.9,1.0])$; as the terms unify, then the following subgoal ($tall(X)$, line 08 and from line 17 to 19) is launched with an approximation degree $D=min([0.2,0.4],[0.9,1.0])$.  Finally, the assignation $X=a$ with $[0.2,0.4]$ is produced.

We have implemented a limit to the expansion of the search space in a computation by what we called a ``$\lambda$-cut for IVFSs''. When the LambdaCutIVFS flag is set to a value different than $[0.0,0.0]$, the weak unification process fails if the computed approximation degree goes below the stored LambdaCutIVFS value. Therefore, the computation also fails and all possible branches starting from that choice point are discarded. By default the LambdaCutIVFS value is $[0.0,0.0]$. However, the lambda cut flag can be set to a different value by means of a $\lambda$-cut directive: ``:-lambdaCutIVFS(N).'', where N is an interval between $[0.0,0.0]$ and $[1.0,1.0]$. For example, a $\lambda$-cut of $[0.5,0.5]$ could be established by using the following directive: ``:-lambdaCutIVFS([0.5,0.5])".


\section{Applications}
\label{sc:applications}
The main realms for the application of the IVFSs programming language described in this paper are those which involve natural language semantics processing. In this section, we will discuss two of them: linguistic knowledge modelling and proximity-based logic programming using linguistic resources.

\subsection{Linguistic Knowledge Modelling}
Linguistic knowledge modelling handles the computational representation of knowledge that is embedded in natural language. This framework can be enhanced by combining multiadjoint paradigm with interval-valued fuzzy sets \cite{Medina10}. For example, we can define interval-valued annotated atoms. Let us assume the same definition of suitable journal given in \cite{Medina09}, that is, a journal with a high impact factor, a medium immediacy index, a relatively big half-life and with a not bad position in the listing of the category. Now, we introduce in the program the following inference rule:

{\small
\begin{verbatim}
suitable_journal(X):-impact_factor(X)[0.8,0.9], 
                     immediacy_index(X)[0.4,0.6], 
                     cited_half_life(X)[0.6,0.7], 
                     best_position(X)[0.4,0.6].
\end{verbatim}
}

Now, let us suppose the \textit{IEEE Transactions of Fuzzy System} journal has the following 
properties: ``high'' impact factor, ``small'' immediacy index, ``relatively small'' cited half 
life and the ``best position''. Regarding the linguistic variables: 
``high'', ``medium'', ``relatively big'' and ``not bad'', which can be related to the following  
truth-values: $[0.8,0.9]$, $[0.4,0.6]$, $[0.6,0.7]$ and $[0.4,0.6]$, respectively, considering  
the variables ``medium'' and ``not a bad'' with a similar meaning. This knowledge could be model in 
an interval-valued fuzzy logic language as follows:

{\small
\begin{verbatim}
%% high impact factor
impact_factor(ieee_fs)[0.8,0.9]   

%% small immediacy index
immediacy_index(ieee_fs) [0.3,0.5]

%% relatively small
cited_half_life(ieee_fs)[0.3,0.5] 

%% best position
best_position(ieee_fs) [1,1]  
\end{verbatim}
}

When the query ``suitable\_journal(X)'' is launched, then the system answers: $``X=ieee\_fs'' with [0.3,0.6]$.

\subsection{Proximity-based Logic Programming based on WordNet} 

Proximity-based Logic Programming is a framework that provides us with the capability of enriching semantically classical logic programming languages by using Proximity Equations (PEs). A limitation of this approach is that PEs are mostly defined for a specific domain \cite{Jul16,RR10}, being the designer who manually fixes the values of these equations. This fact makes harder to use PLP systems in real applications. 

A possible solution consists in obtaining the proximity equations from WordNet which requires to employ interval-valued fuzzy sets in order to deal with the high uncertainty generated by the possibility of using several different semantic similarity metrics. Let us assume a fragment of a deductive database that stores information about people and their preferences. The proximity equations can be generated from WordNet, we only put here some of them (see \cite{RM16} for more detail).

{\footnotesize
\begin{verbatim}
%% m loves mountaineering   
loves(mary,mountaineering). 

%% j likes football
likes(john,football). 

%% peter  plays basketball 
plays(peter,basketball).

%% if a person practises sports 
%% the he/she is a healthy person
healthy(X):- practices(X,sport).

%%automatically generated from wordnet
love~passion=[0.25,0.8]. 
basketball~hoops=[1,1].    
play~act=[0.25,0.7].     
practice~rehearse=[1,1]. 
sport~variation=[0.1,0.5]. 
sport~fun=[0.3,0.8].                             
\end{verbatim}
}

\section{Related Work}
\label{sc:rc}
In the literature, other proposals that address our same goal can be found~\cite{AG93,GMV04}. One of the most relevant 
ones is Ciao-Prolog~\cite{GMV04} and, for that reason, we will analyse in detail the differences between it and Bousi-Prolog 
in order to clarify and reinforce the novelty of our proposal:

\begin{itemize}
\item \textbf{From the point of view of its implementation.} In Ciao-Prolog, IVFSs are included by means of constrains 
and hence a translator must be implemented. As a result, the programmer must code the variables in order to manage 
the truth values and get the answers from the system based on those constraints. In Bousi-Prolog, on the other 
hand, IVFSs are included in a different way, where the compiler and the warren abstract machine are enhanced 
by using a IVFSs data structure which has been created and adapted for this architecture. As a result, intervals 
work as a standard data structure in the code of the program instead of a particular set of variables 
defined \emph{ad hoc} by the programmer. This feature allows us to include IVFSs in both fuzzy unification 
(see \cite{RM16}) and fuzzy resolution. In addition, this framework also allows other possible extensions, such as 
the incorporation of a reasoning module using WordNet (see \cite{RM16}).  

\item \textbf{From the point of view of its syntax.} Ciao-Prolog and Bousi-Prolog, although both are Prolog 
languages, they have a well differentiated syntax. The former only allows the annotation of facts, rules cannot 
be annotated because these only allow the use of an aggregator operator for the computing of the 
annotated IVFSs. The latter, on the other hand, allows the user both the annotation the fact and rules 
by means of IVFSs. In addition, if we focus on the inference engine, while Ciao-Prolog only extends 
the resolution mechanism, Bousi-Prolog uses interval-valued proximity 
equations (e.g., “young~teenager=[0.6,0.8]”), which extends both the resolution 
and unification process. 

\item \textbf{From the point of view of its semantics.} Ciao-Prolog and Bousi-Prolog have relevant differences at the 
semantic levels as well. Firstly, Bousi-Prolog implements the concept of cut-level, which allows to the user 
imposes a threshold in the system, and according to it you can be as precise as you want in your answer. This 
is a substantial change due to the introduction of a threshold operational semantics. Therefore, our operational 
mechanism behaves very much as the one of a Prolog system (obtaining correct answers one by one), while this option 
is not available in Ciao semantics. As we mentioned in section~\ref{implementation}, a $\lambda$-cut for IVFS approximation 
degrees has been implemented. The concepts of interpretation, least model semantics, model, so on, are presented and 
defined in a different way, in Bousi-Prolog the operational semantics is based an 
extension of SLD Resolution. In~\cite{GMV04} the type of resolution is based on the classical SLD 
Resolution of Prolog Systems.
\end{itemize}

\section{Conclusions and future work}
\label{sc:conclusions}
We have formally defined and efficiently implemented a simple interval-valued fuzzy programming language using interval-valued fuzzy sets for modelling the uncertainty and imprecision of the knowledge associated to lexical resources. As future work, we propose to extend our language and to provide results of soundness and completeness. Additionally, we want to develop a fully integrated framework in which interval-valued fuzzy sets and interval-valued fuzzy relations can be combined in a same 
framework.

\section*{Acknowledgements}
{\footnotesize
The authors gratefully acknowledges the comments made by reviewers. This work has been partially supported 
by FEDER and the State Research Agency (AEI) of the Spanish Ministry of Economy and Competition under 
grants TIN2016-76843-C4-2-R (AEI/FEDER, UE) and TIN2014-56633-C3-1-R, the Conseller\'ia  
de  Cultura, Educaci\'on  e  Ordenaci\'on  Universitaria  (the Postdoctoral Training Grants 2016 and Centro 
singular de investigaci\'on de Galicia accreditation 2016-2019, ED431G/08) and European  
Regional  Development Fund (ERDF). This work has been done in collaboration with the research group 
SOMOS (SOftware-MOdelling-Science) funded by the Research Agency and the Graduate School of Management 
of the B\'io-B\'io University.}

\end{document}